\documentclass{article}

\usepackage[preprint]{neurips_2026}

\usepackage[utf8]{inputenc}
\usepackage[T1]{fontenc}
\usepackage{hyperref}
\usepackage{url}
\usepackage{booktabs}
\usepackage{amsfonts}
\usepackage{amsmath}
\usepackage{amssymb}
\usepackage{amsthm}
\usepackage{nicefrac}
\usepackage{microtype}
\usepackage{xcolor}
\usepackage{graphicx}
\usepackage{tikz}
\usetikzlibrary{positioning, arrows.meta, fit, backgrounds, shadows.blur}

\newtheorem{theorem}{Theorem}
\newtheorem{proposition}{Proposition}
\newtheorem{corollary}{Corollary}
\newtheorem{lemma}{Lemma}
\newtheorem{assumption}{Assumption}

\newcommand{\x}{\mathbf{x}}
\newcommand{\y}{\mathbf{y}}

\newcommand{\E}{\mathbb{E}}
\newcommand{\eu}{\mathbf{e}}
\newcommand{\proj}{\mathbf{A}}
\newcommand{\filter}{\mathbf{K}}
\newcommand{\weights}{\mathbf{W}}

\title{A Deep Risk Estimator for Known Operator Learning}

\author{%
  Andreas Maier\thanks{Corresponding author: \texttt{andreas.maier@fau.de}.} \quad
  Md Hasan \quad
  Paulina Conrad \quad
  Paula Andrea P\'erez Toro \\[2pt]
  Pattern Recognition Lab\\
  Friedrich-Alexander-Universit\"at Erlangen-N\"urnberg\\
  Erlangen, Germany%
}

\begin{document}

\maketitle

\begin{abstract}
We describe an approach for estimating the statistical risk of deep networks that contain a mix of learned and known operators. Building on the maximal training error bounds previously established for known operator learning, we derive a deep risk estimator that connects the expected error of a layered network to the size of the training sample. The estimator decomposes the total risk into a sum over learned layers; every known operator contributes zero to this sum, while every learned layer adds an approximation term inspired by Barron's classic work and an estimation term that decreases with the number of training samples. We are able to show that the bound shrinks whenever a learned layer is replaced by a known operator and that the corresponding sample requirement scales with the number of trainable parameters of the layer that is replaced. As an application, we use computed tomography as an example and compare an operator-aware filtered backprojection network with a fully connected substitute that collapses the entire reconstruction pipeline into a single learned dense matrix. The predicted parameter ratio coincides with the structural sparsity that the analytic decomposition into a circulant filter and a sparse backprojection exposes. We confirm the predicted scaling on CPU at small image scale and on GPU at medium image scale, all on the same scaling law. Beyond CT reconstruction, the estimator applies to physics-informed neural networks that hardcode a known physical operation in its architecture, and we expect the result to be of interest for a broad community working on operator-aware deep learning. Calibrating the per-layer constants on each sweep yields a bound that tracks the empirical test MSE, so the estimator can be used to predict how many training samples are required to reach a target error. Source code, configuration files, and the full-resolution training pipeline are publicly available at \url{https://github.com/akmaier/KnownOperatorCT}.
\end{abstract}

\section{Introduction}

Pattern analysis and machine intelligence have largely been driven by tasks that mimic perceptual problems and are formulated as classification or regression on man-made labels \cite{Maier2019}. With the rise of deep learning, the same methods are now also applied to physics-driven problems such as computed tomography, where parts of the processing chain are not learned but rather analytically known and differentiable \cite{Wuerfl2018,Syben2019pyronn}. The general idea of \emph{known operator learning} is to embed such operators directly into trainable pipelines instead of relearning them from data \cite{Maier2019,maier2022known}. The same idea underlies the \emph{hybrid physics-informed neural network} (PINN) literature \cite{RaissiEtAl2019PINN,KarniadakisEtAl2021PIML}, where governing equations, conservation laws, or known physical operators are hardcoded into the network architecture (rather than only enforced via a soft partial differential equation (PDE) residual loss term as in the canonical PINN formulation), and where the sample efficiency of the resulting hybrid models is central to their appeal.

\citet{Maier2019} have shown that incorporating known operators reduces the maximum training error bound of deep networks. While maximum error bounds capture worst-case behavior, they do not directly speak to the practical question of how the test error of an operator-aware network depends on the size of the training set. This question, however, is at the heart of the appeal of known operator learning: practitioners observe that operator-aware networks often work well with much less data than purely black-box counterparts, and a theoretical statement that captures this behavior is needed.

In this paper, we derive a deep statistical risk estimator that connects the expected squared error of a layered network with known and learned operators to the size of its training sample. We formulate the estimator at the level of the network composition rather than for individual layers, and we use Barron-type per-layer risks as the only assumption on each learned layer. The resulting bound is a sum of layer terms; known operator layers contribute zero, while learned layers contribute an approximation term and an estimation term that decreases with the number of samples. To the knowledge of the authors, this is the first deep risk estimator that targets known operator networks specifically and that yields an explicit sample size dependence. The estimator is derived in this paper from first principles; the proof is given in full in Appendix~\ref{app:proof}. Because the assumptions are layered Lipschitz continuity and a per-layer Barron-type bound on the learned blocks only, the same statement applies directly to any hybrid PINN that hardcodes a physical operation in its architecture.

We illustrate the theorem on the computed tomography (CT) application of \citet{Maier2019} and \citet{Wuerfl2018}, where an operator-aware filtered backprojection network reduces the trainable parameter count by orders of magnitude compared with a purely learned, fully connected substitute. How this dramatic parameter reduction translates into a corresponding reduction in training-data needs has so far been an empirical observation, without an explicit theoretical statement. The deep risk estimator developed in this paper closes that gap, and we confirm the predicted scaling on sample-efficiency sweeps at small, medium, and large image scale.

\paragraph{Contributions.}
This paper makes four contributions.
\begin{enumerate}
\item We state and prove a deep risk estimator for known operator networks that connects the expected squared error of a layered network to the size of its training sample.
\item We show that the estimator decomposes additively over layers and that every known operator removes one full layer term from the bound.
\item We argue that the bound is particularly useful for estimating the training-sample requirements of specific problems, and we calibrate it on a small pilot study to predict required sample sizes for a target error.
\item We apply the result to the CT reconstruction example of \citet{Maier2019} and provide a reproducible surrogate that confirms the predicted low-data advantage of the operator-aware model.
\end{enumerate}

\section{Related Work}

Barron's classical bound separates approximation and estimation for single hidden layer networks and is the closest classical reference for the present work \cite{barron1994approximation}. More recent work has studied compositional and adaptive estimation in deep networks, but these results typically do not consider explicitly inserted known operators \cite{SchmidtHieber2020,Chen2024HilbertLadders,LiuChengLiao2025,DanhoferEtAl2025}. In the inverse problem community, operator learning, learned corrections, and stability analyses provide an adjacent theoretical framework \cite{LunzEtAl2021,DeRyckMishra2022,MolinaroEtAl2023,EvangelistaEtAl2025,AlbertiEtAl2026}.

On the application side, deep learning has long been used as a complement to filtered backprojection in CT \cite{deeplearningct,Wuerfl2018}. The known operator perspective of \citet{Maier2019} is particularly relevant here because it interprets each reconstruction step as a layer while preserving its analytic meaning, and it provides the maximum training error analysis that motivates the present deep risk estimator.

\section{Known Operator Networks}
\label{sec:setup}

The general idea of known operator learning is to embed entire analytic operations into a learning problem instead of relearning them from data \cite{Maier2019,maier2022known}. Figure~\ref{fig:precision} shows the resulting compositional structure schematically. We refer to the input of the network as $\x \in \mathbb{R}^{N_D+1}$ in the homogeneous notation of \cite{Maier2019} and write the target map as a layered composition
\begin{equation}
f_L(\x) = u_L\bigl(u_{L-1}(\cdots u_1(\x)\cdots)\bigr),
\label{eq:target}
\end{equation}
where each layer $u_l : \mathcal{D}_{l-1} \rightarrow \mathcal{D}_l$ is a continuous function on compact domains $\mathcal{D}_{l-1}, \mathcal{D}_l$ of the appropriate ambient dimensions, the input lies in $\mathcal{D}_0$, and $f_L(\x) \in \mathcal{D}_L$. The corresponding network is
\begin{equation}
\hat{f}_L(\x) = \hat{u}_L\bigl(\hat{u}_{L-1}(\cdots \hat{u}_1(\x)\cdots)\bigr),
\label{eq:network}
\end{equation}
and we denote the per-layer approximation error by $\eu_l(\x_l) = u_l(\x_l) - \hat{u}_l(\x_l)$ for $\x_l \in \mathcal{D}_{l-1}$. We say that layer $l$ is a \emph{known operator} if $\hat{u}_l = u_l$ on $\mathcal{D}_{l-1}$, in which case $\eu_l \equiv \mathbf{0}$. We say that layer $l$ is \emph{learned} if $\hat{u}_l$ is a trainable function approximator.

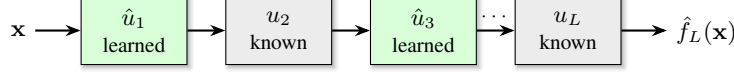
\begin{figure}[t]
\centering
\begin{tikzpicture}[
  every node/.style={font=\small},
  block/.style={draw, rectangle, minimum height=0.9cm, minimum width=1.4cm, align=center, blur shadow={shadow blur steps=4, shadow xshift=1pt, shadow yshift=-1pt, shadow opacity=35}},
  known/.style={block, fill=gray!15},
  learn/.style={block, fill=green!15},
  arrow/.style={-{Stealth[length=2mm]}, thick}
]
\node (xin) {$\x$};
\node[learn, right=0.6cm of xin] (u1) {$\hat{u}_1$\\\footnotesize learned};
\node[known, right=0.5cm of u1] (u2) {$u_2$\\\footnotesize known};
\node[learn, right=0.5cm of u2] (u3) {$\hat{u}_3$\\\footnotesize learned};
\node[known, right=0.5cm of u3] (uL) {$u_L$\\\footnotesize known};
\node[right=0.6cm of uL] (yout) {$\hat{f}_L(\x)$};
\draw[arrow] (xin) -- (u1);
\draw[arrow] (u1) -- (u2);
\draw[arrow] (u2) -- (u3);
\draw[arrow] (u3) -- node[above, font=\scriptsize]{$\cdots$} (uL);
\draw[arrow] (uL) -- (yout);
\end{tikzpicture}
\caption{Known operator networks blend analytic, fixed layers (gray) with trainable layers (green) inside a single compositional pipeline. The minimal requirement on a known operator is the existence of a sub-gradient with respect to its inputs, which makes the operator compatible with backpropagation \cite{Maier2019}.}
\label{fig:precision}
\end{figure}

We make two assumptions on the layers and on the per-layer approximation error of the learned blocks; both are standard in the analysis of compositional function approximators.

\begin{assumption}[Layer Lipschitz continuity and codomain invariance]
\label{ass:lip}
Each layer $u_l$ of the target function~\eqref{eq:target} is Lipschitz continuous with constant $\ell_l \geq 0$ on $\mathcal{D}_{l-1}$. Each learned layer $\hat{u}_l$ is Lipschitz continuous on $\mathcal{D}_{l-1}$ and satisfies $\hat{u}_l(\mathcal{D}_{l-1}) \subseteq \mathcal{D}_l$, so the composition $\hat{f}_L$ in Eq.~\eqref{eq:network} is well defined for every $\x \in \mathcal{D}_0$.
\end{assumption}

\begin{assumption}[Per-layer Barron-type risk]
\label{ass:barron}
Let $\hat{u}_l$ denote the converged learned layer obtained from a training procedure on $N$ samples. There exist a layer complexity $C_l \geq 0$, an approximation width $n_l \geq 1$, an effective number of trainable parameters $p_l \geq 1$, and a constant $\kappa_l > 0$ such that the expected squared layer error, taken at test time over inputs $\mathbf{X}_l$ drawn i.i.d.\ from a fixed evaluation distribution on $\mathcal{D}_{l-1}$, satisfies
\begin{equation}
\E_{\mathbf{X}_l} \, \lVert \eu_l(\mathbf{X}_l) \rVert_2^2
\;\leq\;
\frac{C_l^2}{n_l} + \kappa_l \, \frac{p_l \log N}{N}.
\label{eq:layer-risk}
\end{equation}
The i.i.d.\ condition therefore applies to the test-time inputs of the converged layer; it does not constrain the training trajectory.
\end{assumption}

Assumption~\ref{ass:barron} is satisfied by the classical Barron bound for sigmoidal single hidden layer networks \cite{barron1994approximation} and by several modern deep extensions \cite{SchmidtHieber2020,Chen2024HilbertLadders,LiuChengLiao2025}. We treat it here as the only requirement on each individual learned block; the deep risk estimator that follows is agnostic to which per-layer bound is chosen, as long as it has the form~\eqref{eq:layer-risk}.

\paragraph{Notational note.}
Barron's original statement \cite{barron1994approximation} considers a single hidden layer with $n_l$ sigmoidal units and identifies the number of trainable parameters with the network width, $p_l = n_l$. The layer risk is expressed in $\mathcal{O}$ notation, which omits multiplication with positive constants and lower-order terms in $N$, $n_l$, and $p_l$. In Eq.~\eqref{eq:layer-risk} we separate the width $n_l$ from the parameter count $p_l$ so that compositional layers with $p_l \neq n_l$ (for example a structurally constrained learned layer with shared weights) fit the same template, and we absorb every other constant into $C_l$ and $\kappa_l$. The ratio $p_l / n_l$ is a constant for any given compositional layer instance (it is determined by the layer's architecture, not by $N$). In particular, $\kappa_l$ carries the input dimension of the layer and the multiplicative constants implicit in the $\mathcal{O}$ notation. This separation slightly changes the interpretation of $C_l$ and $\kappa_l$ relative to Barron's: under $n_l = p_l$ the parameter-sharing constraint of a structurally constrained layer would have to be packed into a per-architecture Barron complexity $C_l$, whereas allowing $p_l \neq n_l$ moves the constraint into the explicit parameter count and leaves $C_l$ to encode only the function-class complexity.

\section{A Deep Risk Estimator for Known Operator Learning}
\label{sec:theory}

We are now interested in lifting the per-layer risk~\eqref{eq:layer-risk} to a risk bound on the whole network~\eqref{eq:network}. To do so, we first state how the network error decomposes layerwise and then apply Assumption~\ref{ass:barron} to the learned layers.

\begin{theorem}[Deep Risk Estimator]
\label{thm:deep-risk}
Let $f_L$ and $\hat{f}_L$ be defined by~\eqref{eq:target} and~\eqref{eq:network} on a compact domain $\mathcal{D} \subset \mathbb{R}^{N_D+1}$ and assume Assumption~\ref{ass:lip}. Then for every input distribution on $\mathcal{D}$ the expected squared network error satisfies
\begin{equation}
\E_{\mathbf{X}} \, \lVert f_L(\mathbf{X}) - \hat{f}_L(\mathbf{X}) \rVert_2^2
\;\leq\;
\sum_{l=1}^{L} A_l \, \E_{\mathbf{X}_l} \, \lVert \eu_l(\mathbf{X}_l) \rVert_2^2 ,
\label{eq:deep-risk}
\end{equation}
where the amplification factor is
\begin{equation}
A_l =
\begin{cases}
2^{L-1} \prod_{j=2}^{L} \ell_j^2 & l = 1 , \\
2^{L-l+1} \prod_{j=l+1}^{L} \ell_j^2 & 2 \leq l \leq L .
\end{cases}
\label{eq:Al}
\end{equation}
If, in addition, every learned layer satisfies Assumption~\ref{ass:barron}, then
\begin{equation}
\E_{\mathbf{X}} \, \lVert f_L(\mathbf{X}) - \hat{f}_L(\mathbf{X}) \rVert_2^2
\;\leq\;
\sum_{l \in \mathcal{L}} A_l \!\left( \frac{C_l^2}{n_l} + \kappa_l \, \frac{p_l \log N}{N} \right) ,
\label{eq:deep-risk-barron}
\end{equation}
where $\mathcal{L} \subseteq \{1, \dots, L\}$ is the index set of learned layers and known operator layers contribute zero to the sum.
\end{theorem}

The amplification factor $A_l$ depends only on the Lipschitz constants of the layers $u_{l+1}, \dots, u_L$. As in the maximum error analysis of \cite{Maier2019}, this means that the cascade structure of the problem itself decides how strongly an error introduced at layer $l$ affects the network output. Each learned layer contributes one term that is composed of an approximation part $C_l^2 / n_l$ and an estimation part $\kappa_l p_l \log N / N$. Eq.~\eqref{eq:deep-risk-barron} is the key bridge between known operator learning and sample efficiency: it states that the expected error decreases at rate $\log N / N$ in the training sample size $N$, with a constant that depends only on the structure of the learned layers.

The proof is given in Appendix~\ref{app:proof}. It proceeds by recursively expanding the network error, using the squared triangle inequality $(a+b)^2 \leq 2a^2 + 2b^2$ at each layer together with the Lipschitz property of the true outer layer, and finally taking expectations.

\begin{corollary}[Known operator reduction]
\label{cor:known}
Replacing learned layer $m$ by a known operator (i.e., setting $\hat{u}_m = u_m$, so $\eu_m \equiv \mathbf{0}$) removes the term $A_m \E\lVert \eu_m \rVert_2^2$ from~\eqref{eq:deep-risk}.
\end{corollary}

Corollary~\ref{cor:known} formalizes the intuition that known operators help by removing learning targets, not by improving them. It is important to highlight that the corollary applies to both the approximation and the estimation contribution of the replaced layer; in particular, both vanish in layers with known operations.

\paragraph{Sample-size estimation for specific problems.}
The deep risk estimator is particularly useful when one wants to know how many training samples a specific problem requires to reach a target test error. Inverting Eq.~\eqref{eq:deep-risk-barron} for $N$ given a target $\varepsilon$ turns the bound into a closed-form sample-size predictor. The per-layer constants $C_l, n_l, \kappa_l$ in Eq.~\eqref{eq:deep-risk-barron} are typically not known analytically for the architectures encountered in practice, but they are easy to calibrate on a small pilot study, after which the inverted bound serves as a safe-side estimate of the training-set size that is sufficient to reach a target accuracy. We make this calibration scheme explicit in the next section.

\section{Computed Tomography Application}
\label{sec:ct}

We reinterpret the classical filtered backprojection (FBP) pipeline as a layered network in the sense of Section~\ref{sec:theory}, where each reconstruction step corresponds to a layer. This perspective lets us compare learning only a small correction to a known physical model (the operator-aware network) with learning the full inverse mapping from data (the fully connected network).

In computed tomography, we are interested in computing a reconstruction $\y \in \mathbb{R}^{H \cdot H}$ from a set of projection images $\x \in \mathbb{R}^{V \cdot B}$ acquired at $V$ views and $B$ detector bins. Both are related by the X-ray transform $\proj$, $\proj \y = \x$. The Moore--Penrose inverse of $\proj$ gives rise to FBP \cite{kak2001principles}
\begin{equation}
\y = \mathrm{ReLU}\bigl( \proj^\top \filter \weights \, \x \bigr) ,
\label{eq:fbp}
\end{equation}
where $\weights$ collects geometry-dependent diagonal weights, $\filter$ implements the reconstruction filter, $\proj^\top$ is the backprojection, and $\mathrm{ReLU}(\cdot)$ suppresses negative values \cite{Wuerfl2018,Maier2019}.

Eq.~\eqref{eq:fbp} is a known operator network with four layers in the sense of Section~\ref{sec:theory}: $u_1 = \weights$, $u_2 = \filter$, $u_3 = \proj^\top$, and $u_4 = \mathrm{ReLU}$. In the operator-aware (KO) model of \citet{Wuerfl2018} and \citet{Maier2019}, only $\weights$ is treated as a learned layer; the remaining three layers are known operators. The fully connected (FC) counterfactual collapses the entire reconstruction pipeline into a single dense learned matrix $\mathbf{M}$ that maps projections directly to image pixels, followed by the same fixed ReLU, $\hat{\y}_{\mathrm{FC}} = \mathrm{ReLU}(\mathbf{M} \, \x)$. Figure~\ref{fig:ct-arch} contrasts the two architectures.

\begin{figure}[t]
\centering
\begin{tikzpicture}[
  every node/.style={font=\small},
  block/.style={draw, rectangle, minimum height=0.95cm, minimum width=1.5cm, align=center, blur shadow={shadow blur steps=4, shadow xshift=1pt, shadow yshift=-1pt, shadow opacity=35}},
  known/.style={block, fill=gray!15},
  learn/.style={block, fill=green!15},
  dense/.style={block, fill=red!15},
  arrow/.style={-{Stealth[length=2mm]}, thick}
]
\node (xa) {$\x$};
\node[learn, right=0.5cm of xa] (Wa) {$\weights$\\\footnotesize learned};
\node[known, right=0.4cm of Wa] (Ka) {$\filter$\\\footnotesize known};
\node[known, right=0.4cm of Ka] (Aa) {$\proj^\top$\\\footnotesize known};
\node[known, right=0.4cm of Aa] (Ra) {ReLU\\\footnotesize known};
\node[right=0.4cm of Ra] (ya) {$\y$};
\node[left=0.05cm of xa, font=\footnotesize\itshape] {KO net};
\draw[arrow] (xa) -- (Wa);
\draw[arrow] (Wa) -- (Ka);
\draw[arrow] (Ka) -- (Aa);
\draw[arrow] (Aa) -- (Ra);
\draw[arrow] (Ra) -- (ya);
\node[above=0.05cm of Wa, font=\scriptsize] {$p_{\mathrm{KO}} = 2.30 \cdot 10^{4}$};
\node (xb) [below=1.5cm of xa] {$\x$};
\node[dense, right=0.5cm of xb] (Mb) {dense $\mathbf{M}$\\\footnotesize learned};
\node[known, right=0.4cm of Mb] (Rb) {ReLU\\\footnotesize known};
\node[right=0.4cm of Rb] (yb) {$\y$};
\node[left=0.05cm of xb, font=\footnotesize\itshape] {FC net};
\draw[arrow] (xb) -- (Mb);
\draw[arrow] (Mb) -- (Rb);
\draw[arrow] (Rb) -- (yb);
\node[above=0.05cm of Mb, font=\scriptsize] {$p_{\mathrm{FC}} = 1.51 \cdot 10^{9}$};
\end{tikzpicture}
\caption{Computed tomography networks compared in this paper. The operator-aware network (top) trains only the diagonal weighting layer $\weights$, while the filter $\filter$, the backprojection $\proj^\top$, and the ReLU layer are known operators. The fully connected counterfactual (bottom) collapses the entire reconstruction pipeline into a single dense learned matrix $\mathbf{M}$ that maps projections directly to image pixels, followed by the same fixed ReLU. Trainable parameter counts above each learned layer correspond to a slice-wise $256 \times 256$ fan-beam discretization with $90$ views and $256$ detector bins.}
\label{fig:ct-arch}
\end{figure}
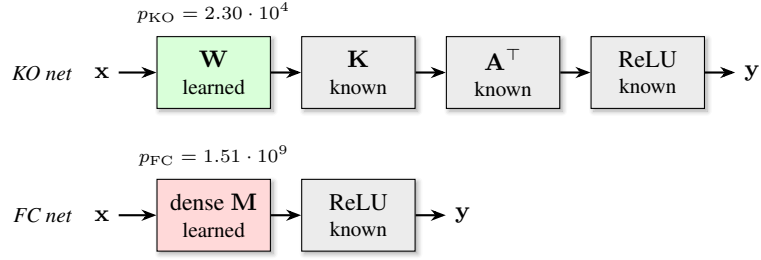

\paragraph{Where the operator-aware reduction comes from.}
We count the trainable parameters of each layer in Eq.~\eqref{eq:fbp}. The reconstruction filter $\filter$ is a one-dimensional circulant convolution applied per view along the detector axis; it is fully specified by its $\mathcal{O}(B)$ frequency-domain coefficients and is hardcoded at run time, so it carries no trainable parameters. The backprojection $\proj^\top$ is sparse because only the rays that intersect a given image pixel contribute to it; like $\filter$ it is hardcoded and carries no trainable parameters. The ReLU layer is parameter-free. The trainable footprint of the operator-aware network is therefore exactly $\weights$, a diagonal projection-domain weighting with one weight per detector bin and angle, $p_{\mathrm{KO}} = V \cdot B$. The dense substitute $\mathbf{M}$ has none of this structure: it is an unconstrained matrix that maps each of the $V \cdot B$ measurements to each of the $H^2$ image pixels, so $p_{\mathrm{FC}} = H^2 \cdot V \cdot B = H^2 \cdot p_{\mathrm{KO}}$. The factor of $H^2$ between the two trainable footprints is the parameter-count component of the structural sparsity exposed by the analytic decomposition. However, it does not yet deliver the complete picture.

We now apply the deep risk estimator of Theorem~\ref{thm:deep-risk} to both networks. In the operator-aware case, the filter, backprojection, and ReLU layers are known operators that contribute zero to Eq.~\eqref{eq:deep-risk-barron}, so the layer-$\weights$ estimation term, amplified by $A_1$, is the network's total estimation contribution (Appendix~\ref{app:ct}); in the fully connected case, only the single dense matrix $\mathbf{M}$ contributes.

\begin{proposition}[Operator-aware versus fully connected sample complexity]
\label{prop:ct-proxy}
Consider the operator-aware filtered backprojection network with $p_{\mathrm{KO}} = V \cdot B$ trainable diagonal weights and the fully connected counterfactual with $p_{\mathrm{FC}} = H^2 \cdot V \cdot B$ trainable weights of the dense matrix $\mathbf{M}$. With the layer amplifications $A_1^{\mathrm{KO}} = 8 \, \lVert \filter \rVert_2^2 \, \lVert \proj^\top \rVert_2^2$ and $A_1^{\mathrm{FC}} = 2$ derived in Appendix~\ref{app:ct}, equating the network-level bounds of Eq.~\eqref{eq:deep-risk-barron} for the two architectures at the same target error $\varepsilon$ yields the sample-complexity double ratio
\begin{equation}
\frac{N_{\mathrm{FC}} / \log N_{\mathrm{FC}}}{N_{\mathrm{KO}} / \log N_{\mathrm{KO}}} \;=\; \frac{A_1^{\mathrm{FC}} \, \kappa_{\mathbf{M}} \, p_{\mathrm{FC}}}{A_1^{\mathrm{KO}} \, \kappa_1 \, p_{\mathrm{KO}}} \;\cdot\; \frac{\varepsilon - A_1^{\mathrm{KO}} \, C_1^2 / n_1}{\varepsilon - A_1^{\mathrm{FC}} \, C_{\mathbf{M}}^2 / n_{\mathbf{M}}} ,
\label{eq:prop2-twofactor}
\end{equation}
valid for any target error $\varepsilon$ above both architectures' approximation budgets $A_1 \, C^2 / n$. The first factor is a constant of the imaging geometry: under matched constants $\kappa_1 = \kappa_{\mathbf{M}}$ it reduces to $H^2 / (4 \, \lVert \filter \rVert_2^2 \, \lVert \proj^\top \rVert_2^2)$, the parameter-count ratio folded with the operator-norm correction exposed by the analytic FBP decomposition. Computing the operator norms numerically is expensive but the ratio is fixed once the geometry is set. The second factor depends on the target error and on each architecture's approximation budget $A_1 \, C^2 / n$, and is unity exactly when the two budgets match.
\end{proposition}

\paragraph{Structural sparsity sets both factors.}
The same analytic decomposition into a circulant filter and a sparse backprojection produces both terms in Eq.~\eqref{eq:prop2-twofactor}. The KO's diagonal weighting only learns a residual correction over the analytic FBP, so its approximation error $C_1^2/n_1$ captures the irreducible noise plus a small structural residual; the dense FC, in contrast, must learn the complete inverse mapping from projections to image pixels through a single matrix, so its approximation error $C_{\mathbf{M}}^2/n_{\mathbf{M}}$ also includes the entire structural complexity of the inverse problem that the operator-aware decomposition removes by construction. Both the parameter-count component of the structural-sparsity factor and the approximation-budget gap of the second factor are therefore consequences of the same structural sparsity, and both contribute multiplicatively to the dense network's sample-complexity penalty. Table~\ref{tab:fullscale} reports the parameter-count component across $H \in \{8, 16, 32, 128, 256, 512\}$.

\paragraph{The bound, regrouped for sample-budget use.}
The proxy $\log N / N$ above tracks how the estimation term shrinks with $N$ but does not yet have absolute units. To turn Eq.~\eqref{eq:deep-risk-barron} into a calibrated risk predictor, we regroup the per-layer constants $C_l$ and $\kappa_l$ from Assumption~\ref{ass:barron} into two empirically observable numbers per learned layer:
\begin{equation}
\E \, \lVert \y - \hat{\y} \rVert_2^2 \;\approx\; \underbrace{\mathrm{floor}_l}_{\text{noise} \,+\, A_l \, C_l^2 / n_l} \;+\; \underbrace{\sigma_l}_{A_l \, \kappa_l \, p_l} \cdot \, \frac{\log N}{N} .
\label{eq:calibrated-bound}
\end{equation}
A pilot study gives $\mathrm{floor}_l$ (lowest seen mean squared error (MSE) at large $N$, which absorbs the irreducible data noise and the residual approximation term $A_l C_l^2 / n_l$) and $\sigma_l$ (the slope of the test MSE against $\log N / N$, which absorbs the geometry-dependent prefactor on the estimation term and any architecture-specific scale). These are the only two numbers needed to read a sample budget off the bound for that problem.

\begin{table}[t]
\centering
\small
\caption{Bound-inspired full-scale estimate for the CT application across image scales. Trainable parameter counts grow as $H^2$ between the operator-aware (KO) and fully connected (FC) networks. The $H \in \{8, 16, 32\}$ columns correspond to the CPU surrogate sweeps of Section~\ref{sec:exp}, the $H \in \{128, 256\}$ columns to the GPU sweeps, and the $H = 512$ column to the slice-wise setting of \cite{Wuerfl2018,Maier2019}. Adam state assumes $16$ bytes per parameter (master weights, $m$, $v$, gradient).}
\label{tab:fullscale}
\begin{tabular}{lrrrrrr}
\toprule
Quantity & $H = 8$ & $H = 16$ & $H = 32$ & $H = 128$ & $H = 256$ & $H = 512$ \\
\midrule
$V$, $B$ & $10, 8$ & $20, 16$ & $40, 32$ & $60, 128$ & $90, 256$ & $180, 512$ \\
$p_{\mathrm{KO}}$ & $80$ & $3.20 \cdot 10^{2}$ & $1.28 \cdot 10^{3}$ & $7.68 \cdot 10^{3}$ & $2.30 \cdot 10^{4}$ & $9.22 \cdot 10^{4}$ \\
$p_{\mathrm{FC}}$ & $5.12 \cdot 10^{3}$ & $8.19 \cdot 10^{4}$ & $1.31 \cdot 10^{6}$ & $1.26 \cdot 10^{8}$ & $1.51 \cdot 10^{9}$ & $2.42 \cdot 10^{10}$ \\
$p_{\mathrm{FC}} / p_{\mathrm{KO}} = H^2$ & $64$ & $256$ & $1{,}024$ & $1.64 \cdot 10^{4}$ & $6.55 \cdot 10^{4}$ & $2.62 \cdot 10^{5}$ \\
KO FP32 weights & $0.31$ KiB & $1.25$ KiB & $5.00$ KiB & $30.0$ KiB & $90.0$ KiB & $0.35$ MiB \\
KO Adam state & $1.25$ KiB & $5.00$ KiB & $20.0$ KiB & $0.12$ MiB & $0.35$ MiB & $1.41$ MiB \\
FC FP32 weights & $20.0$ KiB & $0.31$ MiB & $5.00$ MiB & $0.47$ GiB & $5.62$ GiB & $90.0$ GiB \\
FC Adam state & $80.0$ KiB & $1.25$ MiB & $20.0$ MiB & $1.88$ GiB & $22.5$ GiB & $360$ GiB \\
\bottomrule
\end{tabular}
\end{table}

\section{Experiments}
\label{sec:exp}

We probe the deep risk estimator at five operating points spanning the surrogate and full-resolution regimes: a fully reproducible CPU surrogate at $H \in \{8, 16, 32\}$ with $V = 1.25\,H$ equally spaced views and $B = H$ detector bins, and two GPU sweeps at $H \in \{128, 256\}$ with $V = 60, 90$ and $B = H$. The CPU sweeps fit both architectures in closed form (ridge-regularised least squares with the regularisation coefficient chosen on a held-out validation set), use $N \in \{4, 8, 16, 32, 64\}$, and average over five seeds. The GPU sweeps train the operator-aware network with stochastic gradient descent and the fully connected baseline with ridge regression on a unified phantom pool, sweep $N \in \{4, 16, 64, 256, 1024, 2048\}$, and average over three seeds. Implementation details for both regimes are documented in Appendix~\ref{app:supp-exp}.

\begin{figure}[t]
\centering
\includegraphics[width=\linewidth]{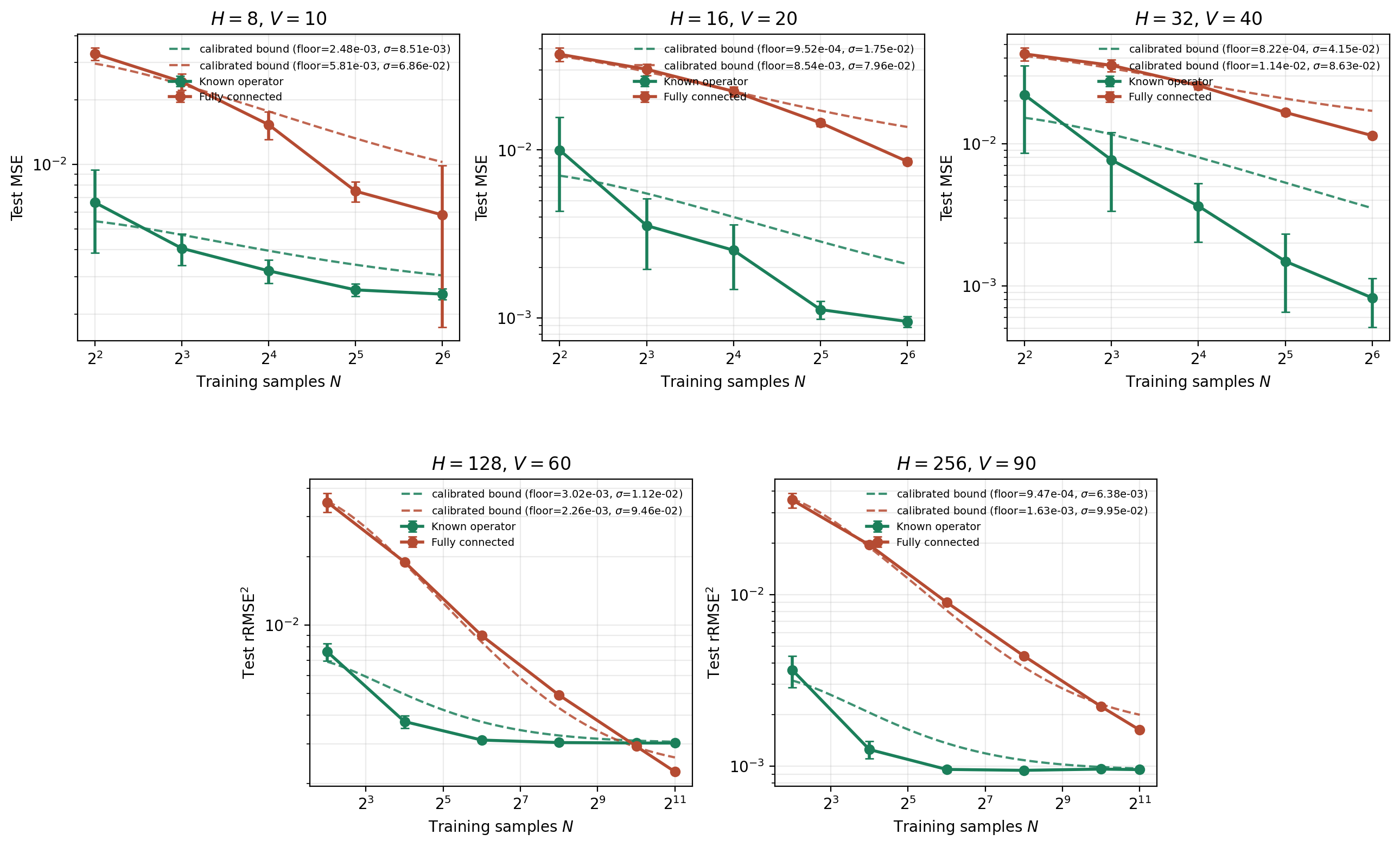}
\caption{Sample-efficiency sweeps at five operating points: CPU surrogate at $H \in \{8, 16, 32\}$ (top row, test mean squared error, five seeds) and GPU sweeps at $H \in \{128, 256\}$ (bottom row, test rRMSE$^2$, three seeds). Each panel overlays the calibrated bound $\mathrm{floor}_l + \sigma_l (\log N / N)$ (dashed) per architecture. The dashed bound is an upper bound on the empirical curve in the data-driven regime; the floor gap between the two dashed lines visualises the approximation-budget factor of Proposition~\ref{prop:ct-proxy}.}
\label{fig:sweeps}
\end{figure}

\paragraph{Calibrated bound matches the data.}
Figure~\ref{fig:sweeps} shows the resulting empirical vs.\ calibrated theoretical risk. The known operator network reaches its asymptotic test error quickly at every scale, while the fully connected baseline either plateaus at a higher floor (CPU sweeps and the largest GPU sweep at $H = 256$) or eventually crosses the operator-aware floor at large $N$ (the GPU sweep at $H = 128$). The two constants per (architecture, sweep) are obtained by a least-squares fit on the same data: $\mathrm{floor}_l$ is the smallest mean error observed in the sweep, and $\sigma_l$ is the no-intercept slope of $(\mathrm{error} - \mathrm{floor}_l)$ against $\log N / N$. The resulting dashed calibrated bound from Eq.~\eqref{eq:calibrated-bound} tracks the empirical curve closely across each panel and clearly upper-bounds it in the low-$N$ regime, where the safe-side character of the bound matters most. Table~\ref{tab:calibration} reports the resulting calibration factors for all five (sweep, architecture) pairs. With a single per-architecture pair of constants, Eq.~\eqref{eq:deep-risk-barron} is predictive of the empirical error on the same problem and can be inverted for a target sample size via Eq.~\eqref{eq:prop2-twofactor}.

\begin{table}[t]
\centering
\small
\caption{Calibration factors $\mathrm{floor}_l$ and $\sigma_l$ from Eq.~\eqref{eq:calibrated-bound}, fitted on each sweep by least squares. CPU rows use test MSE; GPU rows use test rRMSE$^2$. Two numbers per (architecture, sweep) drive both factors of Proposition~\ref{prop:ct-proxy}.}
\label{tab:calibration}
\begin{tabular}{@{}lcccc@{}}
\toprule
Sweep & $\mathrm{floor}_{\mathrm{KO}}$ & $\sigma_{\mathrm{KO}}$ & $\mathrm{floor}_{\mathrm{FC}}$ & $\sigma_{\mathrm{FC}}$ \\
\midrule
$H = 8$ (CPU)    & $2.48 \cdot 10^{-3}$ & $8.51 \cdot 10^{-3}$ & $5.81 \cdot 10^{-3}$ & $6.86 \cdot 10^{-2}$ \\
$H = 16$ (CPU)   & $9.52 \cdot 10^{-4}$ & $1.75 \cdot 10^{-2}$ & $8.54 \cdot 10^{-3}$ & $7.96 \cdot 10^{-2}$ \\
$H = 32$ (CPU)   & $8.22 \cdot 10^{-4}$ & $4.15 \cdot 10^{-2}$ & $1.14 \cdot 10^{-2}$ & $8.63 \cdot 10^{-2}$ \\
$H = 128$ (GPU)  & $3.02 \cdot 10^{-3}$ & $1.11 \cdot 10^{-2}$ & $2.26 \cdot 10^{-3}$ & $9.45 \cdot 10^{-2}$ \\
$H = 256$ (GPU)  & $9.47 \cdot 10^{-4}$ & $6.32 \cdot 10^{-3}$ & $1.63 \cdot 10^{-3}$ & $9.94 \cdot 10^{-2}$ \\
\bottomrule
\end{tabular}
\end{table}

\paragraph{Two factors interact differently across scales.}
Across the five operating points, the structural-sparsity factor $\sigma_{\mathrm{FC}} / \sigma_{\mathrm{KO}}$ ranges over roughly $2$ to $16$ and the approximation-budget factor takes on a regime-dependent role (Table~\ref{tab:calibration}). On the CPU surrogate the FC's floor stays markedly above the KO's at every scale and the approximation-budget factor is therefore always greater than one. On the GPU sweeps the picture is more nuanced: at $H = 128$ the FC is given enough samples ($N = 2048$) that its floor falls slightly below the KO's, so the approximation-budget factor at the FC's eventual floor is below one; at $H = 256$ the FC has not reached the KO floor in the sample range we tested. In every case, the operator-aware network reaches its data-driven plateau much earlier than the dense FC, so for the realistic clinical-CT sample budget ($N$ in the low thousands of slices) inverting Eq.~\eqref{eq:prop2-twofactor} predicts that the FC needs one to two orders of magnitude more samples than the KO to match the KO's accuracy.

\begin{table}[t]
\centering
\small
\caption{Five operating points spanning the surrogate and full-resolution regimes. The parameter ratio $p_{\mathrm{FC}} / p_{\mathrm{KO}} = H^2$ holds at every point and is the parameter-count component of the structural-sparsity factor in Eq.~\eqref{eq:prop2-twofactor}.}
\label{tab:scaling}
\begin{tabular}{@{}lccccl@{}}
\toprule
Operating point & $H$ & $V$ & $B$ & $p_{\mathrm{FC}} / p_{\mathrm{KO}}$ & Outcome \\
\midrule
CPU sweep & $8$   & $10$ & $8$   & $64$        & KO dominates FC at every $N$ \\
CPU sweep & $16$  & $20$ & $16$  & $256$       & KO dominates FC at every $N$ \\
CPU sweep & $32$  & $40$ & $32$  & $1{,}024$   & KO dominates FC at every $N$ \\
GPU sweep & $128$ & $60$ & $128$ & $16{,}384$  & KO dominates at small $N$; FC matches by $N = 2{,}048$ \\
GPU sweep & $256$ & $90$ & $256$ & $65{,}536$  & KO dominates at every $N$ tested \\
\bottomrule
\end{tabular}
\end{table}

\section{Discussion}

The deep risk estimator of Theorem~\ref{thm:deep-risk} connects the expected squared error of a known operator network to the size of its training sample, every known operator simply removes one layer term from the bound, and the additional sample requirement of a fully connected substitute scales with the parameters it introduces. In the CT application, the parameter ratio $p_{\mathrm{FC}}/p_{\mathrm{KO}} = H^2$ has two practical consequences. First, the dense substitute needs $H^2$ times more memory and compute than the operator-aware network to instantiate at all -- this is the operational reason an $H = 512$ FC reconstruction does not fit on accessible single-GPU hardware (Table~\ref{tab:fullscale}). The parameter ratio is therefore a hard memory and compute constraint, not just a sample-complexity quantity. Second, while the sheer magnitude of $p_{\mathrm{FC}}/p_{\mathrm{KO}}$ suggests that the operator-aware network must be vastly more sample-efficient, the empirical sample-complexity gap is captured more tightly by the slope ratio $\sigma_{\mathrm{FC}}/\sigma_{\mathrm{KO}}$ in Eq.~\eqref{eq:prop2-twofactor}, which sits between two and ten across our sweeps (Table~\ref{tab:calibration}). The operator-aware network still wins clearly in the low-data regime -- it reaches its asymptote with roughly an order of magnitude fewer samples than the dense baseline -- but with sufficient data the dense matrix's expressivity catches up and, at the $H = 128$ sparse-view operating point, surpasses the operator-aware floor. The known-operator FBP architecture of \citet{Maier2019} was designed for limited-angle CT, where the missing data is concentrated outside the measured angular range, rather than for sparse-view CT, where the missing data is distributed across all angles; that the dense matrix can find a better solution in the sparse-view regime once given enough data is consistent with this task-domain mismatch. The qualitative reconstructions in Appendix~\ref{sec:qualitative} illustrate both ends of the trade-off.

The present analysis has four limitations. First, Assumption~\ref{ass:barron} requires a Barron-type per-layer bound, and the constants $\kappa_l$ and $C_l$ may differ between blocks. Whether such a bound holds for a particular layer depends on its function class and activation, and characterizing the constants rigorously for non-sigmoidal modern architectures such as deep ReLU networks, transformers, or normalizing flows is itself an open research area; for the diagonal weighting layer in our CT example the bound reduces to the classical $p \log N / N$ rate of linear regression with $p$ trainable parameters, so the Barron-type form holds without invoking the sigmoidal/ReLU activation argument. More broadly, the deep risk estimator is architecture-agnostic given Lipschitz continuity and a per-layer Barron-type bound on the learned blocks: convolutional layers and feedforward attention blocks satisfy these assumptions (the latter once a finite Lipschitz handle on softmax is enforced), while recurrent layers require a correct accounting of the recurrent unrolling, which we leave for future work. Second, the deep risk estimator gives an upper bound, so empirical tightness should not be overstated. The gap between bound and observed risk depends on how well the per-layer constants capture the true layer complexity, and overly optimistic constants can produce a tight-looking bound that does not generalize beyond the calibrated regime. Third, the closed-form reading of the structural-sparsity factor as $H^2 / (4 \, \lVert \filter \rVert_2^2 \, \lVert \proj^\top \rVert_2^2)$ in Proposition~\ref{prop:ct-proxy} assumes matched Barron-type constants $\kappa_1 = \kappa_{\mathbf{M}}$ between the operator-aware and the fully connected layer. In practice the two architectures need not share the same $\kappa$, and the empirical slopes $\sigma_l$ in Table~\ref{tab:calibration} absorb whatever mismatch is present together with the amplification factor $A_l$; the closed-form expression should therefore be read as a matched-$\kappa$ reference rather than a tight numerical prediction. The two-factor decomposition itself, by contrast, retains the empirical approximation-budget difference between the two architectures and does not assume a matched approximation budget. Fourth, Assumption~\ref{ass:barron}'s i.i.d.\ condition refers to the test-time evaluation distribution of the converged learned layer; in pipelines with multiple consecutive learned layers trained end-to-end, the intermediate input distributions evolve during optimization, and the present analysis applies after convergence rather than to the training trajectory. The analysis does not replace stability or uncertainty analyses for safety-critical deployment \cite{EvangelistaEtAl2025}. Although only the CT application is shown here, the deep risk estimator requires only layered Lipschitz continuity and a per-layer Barron-type bound for the learned blocks, so the same analysis applies to hybrid PINNs that hardcode analytic operators in the forward pass, motion-compensated reconstruction, computational magnetic resonance imaging, and computer-vision pipelines that embed analytic operators.

\section{Conclusion}

We extend the maximum training error analysis of \cite{Maier2019} to a deep risk estimator that connects the expected error of a known operator network to its training sample size. Every known operator removes one layer term from the bound. The resulting CT comparison decomposes into a structural-sparsity factor and an approximation-budget factor (Proposition~\ref{prop:ct-proxy}). The parameter ratio $p_{\mathrm{FC}}/p_{\mathrm{KO}} = H^2$ is a hard memory and compute multiplier; the $H = 512$ dense substitute does not fit on accessible single-GPU hardware. It is, however, a poor sample-budget predictor on its own. The empirical slope ratio $\sigma_{\mathrm{FC}}/\sigma_{\mathrm{KO}}$ obtained from a small pilot study is a much tighter estimate of how much extra training data the dense network actually needs. Five sweeps at $H \in \{8, 16, 32, 128, 256\}$ confirm that the calibrated bound tracks the empirical test error closely and upper-bounds it in the low-data regime. The operator-aware network reaches its asymptote with roughly an order of magnitude fewer samples than the dense baseline. The dense baseline catches up at large $N$ on the sparse-view operating point at $H = 128$, consistent with the architectural mismatch between the limited-angle target of \citet{Maier2019} and our sparse-view experiments. Implementation details, hardware notes, and qualitative reconstructions are given in Appendix~\ref{app:supp-exp}; source code, configuration files, and the full-resolution GPU pipeline are publicly available at \url{https://github.com/akmaier/KnownOperatorCT}. The next step is to extend the bound-centered analysis to additional hybrid-PINN families that hardcode analytic operators in the forward pass.

\section*{Broader Impact}

Positive impact comes from making reconstruction systems more data efficient and more interpretable, especially in data-scarce medical imaging settings. Beyond the CT example used in this paper, known operators are now heavily used in hybrid PINNs, where governing equations, conservation laws, or known physical operators are hardcoded directly into the network architecture (rather than only enforced via a soft PDE-residual loss as in the canonical PINN formulation of \citet{RaissiEtAl2019PINN}). The deep risk estimator presented here applies to many of these hybrid PINNs, since its only requirements are layered Lipschitz continuity and a per-layer Barron-type bound on the learned blocks. The estimator therefore quantifies, in a sample-complexity sense, why hardcoding a known physical operation reduces the data requirements of the network: every fixed operator simply removes one full term from the bound. The work is therefore expected to be of interest for a broad community working on hybrid PINNs and operator-aware deep learning in scientific computing, computational imaging, and signal processing. Negative impact remains possible if theoretical sample efficiency claims are overgeneralized to safety-critical deployments without stability analysis, uncertainty quantification, or clinical validation. We therefore view the present results as a theoretical and methodological step, not as evidence that low-data CT systems are ready for clinical use.

\section*{Acknowledgments}

This work was supported by German Research Council (DFG) grant MA~4898/30-1.

{
\small
\bibliographystyle{plainnat}
\bibliography{known_operator_refs}

\begin{thebibliography}{18}
\providecommand{\natexlab}[1]{#1}
\providecommand{\url}[1]{\texttt{#1}}
\expandafter\ifx\csname urlstyle\endcsname\relax
  \providecommand{\doi}[1]{doi: #1}\else
  \providecommand{\doi}{doi: \begingroup \urlstyle{rm}\Url}\fi

\bibitem[Alberti et~al.(2026)Alberti, De~Vito, Helin, Lassas, Ratti, and
  Santacesaria]{AlbertiEtAl2026}
Giovanni~S. Alberti, Ernesto De~Vito, Tapio Helin, Matti Lassas, Luca Ratti,
  and Matteo Santacesaria.
\newblock Learning sparsity-promoting regularizers for linear inverse problems.
\newblock \emph{SIAM Journal on Mathematics of Data Science}, 8\penalty0
  (1):\penalty0 167--199, 2026.
\newblock \doi{10.1137/24M1719785}.

\bibitem[Barron(1994)]{barron1994approximation}
Andrew~R. Barron.
\newblock Approximation and estimation bounds for artificial neural networks.
\newblock \emph{Machine Learning}, 14\penalty0 (1):\penalty0 115--133, 1994.

\bibitem[Chen(2024)]{Chen2024HilbertLadders}
Zhengdao Chen.
\newblock Neural {Hilbert} ladders: Multi-layer neural networks in function
  space.
\newblock \emph{Journal of Machine Learning Research}, 25\penalty0
  (109):\penalty0 1--65, 2024.

\bibitem[Danhofer et~al.(2025)Danhofer, D'Ascenzo, Dubach, and
  Poggio]{DanhoferEtAl2025}
David~A. Danhofer, Davide D'Ascenzo, Rafael Dubach, and Tomaso~A. Poggio.
\newblock Position: A theory of deep learning must include compositional
  sparsity.
\newblock In \emph{Proceedings of the 42nd International Conference on Machine
  Learning}, volume 267 of \emph{Proceedings of Machine Learning Research},
  pages 81199--81210, 2025.

\bibitem[De~Ryck and Mishra(2022)]{DeRyckMishra2022}
Tim De~Ryck and Siddhartha Mishra.
\newblock Generic bounds on the approximation error for physics-informed (and)
  operator learning.
\newblock In \emph{Advances in Neural Information Processing Systems},
  volume~35, pages 10945--10958, 2022.

\bibitem[Evangelista et~al.(2025)Evangelista, Loli~Piccolomini, Morotti, and
  Nagy]{EvangelistaEtAl2025}
Davide Evangelista, Elena Loli~Piccolomini, Elena Morotti, and James~G. Nagy.
\newblock To be or not to be stable, that is the question: Understanding neural
  networks for inverse problems.
\newblock \emph{SIAM Journal on Scientific Computing}, 47\penalty0
  (1):\penalty0 C77--C99, 2025.
\newblock \doi{10.1137/23M1586872}.

\bibitem[Kak and Slaney(2001)]{kak2001principles}
Avinash~C. Kak and Malcolm Slaney.
\newblock \emph{Principles of Computerized Tomographic Imaging}.
\newblock SIAM, 2001.

\bibitem[Karniadakis et~al.(2021)Karniadakis, Kevrekidis, Lu, Perdikaris, Wang,
  and Yang]{KarniadakisEtAl2021PIML}
George~Em Karniadakis, Ioannis~G. Kevrekidis, Lu~Lu, Paris Perdikaris, Sifan
  Wang, and Liu Yang.
\newblock Physics-informed machine learning.
\newblock \emph{Nature Reviews Physics}, 3\penalty0 (6):\penalty0 422--440,
  2021.
\newblock \doi{10.1038/s42254-021-00314-5}.

\bibitem[Liu et~al.(2025)Liu, Cheng, and Liao]{LiuChengLiao2025}
Hao Liu, Jiahui Cheng, and Wenjing Liao.
\newblock Deep neural networks are adaptive to function regularity and data
  distribution in approximation and estimation.
\newblock \emph{Journal of Machine Learning Research}, 26\penalty0
  (213):\penalty0 1--56, 2025.

\bibitem[Lunz et~al.(2021)Lunz, Hauptmann, Tarvainen, Sch{\"o}nlieb, and
  Arridge]{LunzEtAl2021}
Sebastian Lunz, Andreas Hauptmann, Tanja Tarvainen, Carola-Bibiane
  Sch{\"o}nlieb, and Simon Arridge.
\newblock On learned operator correction in inverse problems.
\newblock \emph{SIAM Journal on Imaging Sciences}, 14\penalty0 (1):\penalty0
  92--127, 2021.
\newblock \doi{10.1137/20M1338460}.

\bibitem[Maier et~al.(2022)Maier, K{\"o}stler, Heisig, Krauss, and
  Yang]{maier2022known}
Andreas Maier, Harald K{\"o}stler, Marco Heisig, Patrick Krauss, and Seung~Hee
  Yang.
\newblock Known operator learning and hybrid machine learning in medical
  imaging---a review of the past, the present, and the future.
\newblock \emph{Progress in Biomedical Engineering}, 4\penalty0 (2):\penalty0
  022002, 2022.

\bibitem[Maier et~al.(2019)Maier, Syben, Stimpel, W{\"u}rfl, Hoffmann,
  Schebesch, Fu, Mill, Kling, and Christiansen]{Maier2019}
Andreas~K. Maier, Christopher Syben, Bernhard Stimpel, Tobias W{\"u}rfl, Mathis
  Hoffmann, Frank Schebesch, Wei Fu, Lukas Mill, Leonhard Kling, and Silke
  Christiansen.
\newblock Learning with known operators reduces maximum training error bounds.
\newblock \emph{Nature Machine Intelligence}, 1:\penalty0 373--380, 2019.
\newblock \doi{10.1038/s42256-019-0077-5}.

\bibitem[Molinaro et~al.(2023)Molinaro, Yang, Engquist, and
  Mishra]{MolinaroEtAl2023}
Roberto Molinaro, Yunan Yang, Bj{\"o}rn Engquist, and Siddhartha Mishra.
\newblock Neural inverse operators for solving {PDE} inverse problems.
\newblock In \emph{Proceedings of the 40th International Conference on Machine
  Learning}, volume 202 of \emph{Proceedings of Machine Learning Research},
  pages 25105--25139, 2023.

\bibitem[Raissi et~al.(2019)Raissi, Perdikaris, and
  Karniadakis]{RaissiEtAl2019PINN}
Maziar Raissi, Paris Perdikaris, and George~Em Karniadakis.
\newblock Physics-informed neural networks: A deep learning framework for
  solving forward and inverse problems involving nonlinear partial differential
  equations.
\newblock \emph{Journal of Computational Physics}, 378:\penalty0 686--707,
  2019.
\newblock \doi{10.1016/j.jcp.2018.10.045}.

\bibitem[Schmidt-Hieber(2020)]{SchmidtHieber2020}
Johannes Schmidt-Hieber.
\newblock Nonparametric regression using deep neural networks with {ReLU}
  activation function.
\newblock \emph{The Annals of Statistics}, 48\penalty0 (4), 2020.
\newblock \doi{10.1214/19-AOS1875}.

\bibitem[Syben et~al.(2019)Syben, Michen, Stimpel, Seitz, Ploner, and
  Maier]{Syben2019pyronn}
Christopher Syben, Markus Michen, Bernhard Stimpel, Stephan Seitz, Stefan~B.
  Ploner, and Andreas~K. Maier.
\newblock Technical note: {PYRO-NN}: Python reconstruction operators in neural
  networks.
\newblock \emph{Medical Physics}, 46\penalty0 (11):\penalty0 5110--5115, 2019.
\newblock \doi{10.1002/mp.13753}.

\bibitem[W{\"u}rfl et~al.(2016)W{\"u}rfl, Ghesu, Christlein, and
  Maier]{deeplearningct}
Tobias W{\"u}rfl, Florin~C. Ghesu, Vincent Christlein, and Andreas Maier.
\newblock Deep learning computed tomography.
\newblock In \emph{Medical Image Computing and Computer-Assisted Intervention},
  pages 432--440. Springer, 2016.

\bibitem[W{\"u}rfl et~al.(2018)W{\"u}rfl, Hoffmann, Christlein, Breininger,
  Huang, Unberath, and Maier]{Wuerfl2018}
Tobias W{\"u}rfl, Mathis Hoffmann, Vincent Christlein, Katharina Breininger,
  Yixing Huang, Mathias Unberath, and Andreas~K. Maier.
\newblock Deep learning computed tomography: Learning projection-domain weights
  from image domain in limited angle problems.
\newblock \emph{IEEE Transactions on Medical Imaging}, 37\penalty0
  (6):\penalty0 1454--1463, 2018.
\newblock \doi{10.1109/TMI.2018.2833499}.

\end{thebibliography}
}

\appendix

\section{Proof of Theorem~\ref{thm:deep-risk}}
\label{app:proof}

We prove Eq.~\eqref{eq:deep-risk} by induction on $L$. Throughout the proof, we use the squared triangle inequality
\begin{equation}
\lVert \mathbf{a} + \mathbf{b} \rVert_2^2 \;\leq\; 2 \lVert \mathbf{a} \rVert_2^2 + 2 \lVert \mathbf{b} \rVert_2^2 ,
\label{eq:sqtri}
\end{equation}
which follows from the Cauchy--Schwarz inequality, and the Lipschitz property of the true outer layer $u_l$ from Assumption~\ref{ass:lip}, which gives
\begin{equation}
\lVert u_l(\mathbf{a}) - u_l(\mathbf{b}) \rVert_2 \;\leq\; \ell_l \, \lVert \mathbf{a} - \mathbf{b} \rVert_2 \quad \text{for all } \mathbf{a}, \mathbf{b} \in \mathcal{D}_{l-1} .
\label{eq:lip}
\end{equation}
The codomain-invariance part of Assumption~\ref{ass:lip} guarantees that the intermediate representation $\hat{f}_{l-1}(\x)$ produced by the learned subnetwork lies in $\mathcal{D}_{l-1}$, the domain on which $u_l$ is Lipschitz; the Lipschitz assumption on the learned layers $\hat{u}_l$ themselves is used only through this codomain condition and is not invoked elsewhere in the proof.

\paragraph{Notation.}
We write $\boldsymbol{\delta}_l(\x) := f_l(\x) - \hat{f}_l(\x)$ for $l = 0, 1, \dots, L$, with $f_0(\x) = \hat{f}_0(\x) = \x$ and $\boldsymbol{\delta}_0(\x) = \mathbf{0}$. With this notation, the network error of interest is $\boldsymbol{\delta}_L(\x)$.

\paragraph{Recursive identity.}
For $l \geq 1$ we expand
\begin{align}
\boldsymbol{\delta}_l(\x)
&= f_l(\x) - \hat{f}_l(\x) \nonumber \\
&= u_l\bigl(f_{l-1}(\x)\bigr) - \hat{u}_l\bigl(\hat{f}_{l-1}(\x)\bigr) \nonumber \\
&= \underbrace{u_l\bigl(f_{l-1}(\x)\bigr) - u_l\bigl(\hat{f}_{l-1}(\x)\bigr)}_{=: \mathbf{r}_l(\x)}
\;+\; \underbrace{u_l\bigl(\hat{f}_{l-1}(\x)\bigr) - \hat{u}_l\bigl(\hat{f}_{l-1}(\x)\bigr)}_{= \eu_l\bigl(\hat{f}_{l-1}(\x)\bigr)} .
\label{eq:recid}
\end{align}
By the Lipschitz property of the true outer layer $u_l$ in Eq.~\eqref{eq:lip},
\begin{equation}
\lVert \mathbf{r}_l(\x) \rVert_2 \;\leq\; \ell_l \, \lVert \boldsymbol{\delta}_{l-1}(\x) \rVert_2 .
\label{eq:lip-r}
\end{equation}

\paragraph{Squared bound.}
Combining Eq.~\eqref{eq:recid} with the squared triangle inequality~\eqref{eq:sqtri}, we obtain
\begin{align}
\lVert \boldsymbol{\delta}_l(\x) \rVert_2^2
&\leq 2 \lVert \mathbf{r}_l(\x) \rVert_2^2 + 2 \lVert \eu_l(\hat{f}_{l-1}(\x)) \rVert_2^2 \nonumber \\
&\leq 2 \, \ell_l^2 \, \lVert \boldsymbol{\delta}_{l-1}(\x) \rVert_2^2 + 2 \, \lVert \eu_l(\hat{f}_{l-1}(\x)) \rVert_2^2 .
\label{eq:rec-sq}
\end{align}

\paragraph{Induction.}
We prove by induction that for $l = 1, 2, \dots, L$
\begin{equation}
\lVert \boldsymbol{\delta}_l(\x) \rVert_2^2
\;\leq\;
\sum_{k=1}^{l} A_k^{(l)} \, \lVert \eu_k(\hat{f}_{k-1}(\x)) \rVert_2^2
\quad \text{with}\quad
A_k^{(l)} =
\begin{cases}
2^{l-1} \prod_{j=2}^{l} \ell_j^2 & k = 1 , \\
2^{l-k+1} \prod_{j=k+1}^{l} \ell_j^2 & 2 \leq k \leq l ,
\end{cases}
\label{eq:ind}
\end{equation}
where the empty product for $k = l$ equals $1$, so $A_l^{(l)} = 2$ for $l \geq 2$ and $A_1^{(1)} = 1$.

\emph{Base case} ($l = 1$). Since $\boldsymbol{\delta}_0(\x) \equiv \mathbf{0}$ and $\hat{f}_0(\x) = \x$, the recursive identity~\eqref{eq:recid} gives $\boldsymbol{\delta}_1(\x) = \eu_1(\x)$, so $\lVert \boldsymbol{\delta}_1(\x) \rVert_2^2 = \lVert \eu_1(\x) \rVert_2^2 = A_1^{(1)} \lVert \eu_1(\x) \rVert_2^2$ with $A_1^{(1)} = 1$. The squared triangle inequality is not invoked at this step.

\emph{Inductive step.} Assume Eq.~\eqref{eq:ind} holds for $l - 1$. By Eq.~\eqref{eq:rec-sq},
\begin{align}
\lVert \boldsymbol{\delta}_l(\x) \rVert_2^2
&\leq 2 \, \ell_l^2 \, \lVert \boldsymbol{\delta}_{l-1}(\x) \rVert_2^2 + 2 \, \lVert \eu_l(\hat{f}_{l-1}(\x)) \rVert_2^2 \nonumber \\
&\leq 2 \, \ell_l^2 \sum_{k=1}^{l-1} A_k^{(l-1)} \lVert \eu_k(\hat{f}_{k-1}(\x)) \rVert_2^2 + 2 \, \lVert \eu_l(\hat{f}_{l-1}(\x)) \rVert_2^2 \nonumber \\
&= \sum_{k=1}^{l-1} \! \left( 2 \, \ell_l^2 \, A_k^{(l-1)} \right) \lVert \eu_k(\hat{f}_{k-1}(\x)) \rVert_2^2 + 2 \, \lVert \eu_l(\hat{f}_{l-1}(\x)) \rVert_2^2 .
\label{eq:ind-step}
\end{align}
For $2 \leq k \leq l - 1$ we compute $2 \ell_l^2 A_k^{(l-1)} = 2 \ell_l^2 \cdot 2^{(l-1)-k+1} \prod_{j=k+1}^{l-1} \ell_j^2 = 2^{l-k+1} \prod_{j=k+1}^{l} \ell_j^2 = A_k^{(l)}$. For $k = 1$, $2 \ell_l^2 A_1^{(l-1)} = 2 \ell_l^2 \cdot 2^{l-2} \prod_{j=2}^{l-1} \ell_j^2 = 2^{l-1} \prod_{j=2}^{l} \ell_j^2 = A_1^{(l)}$ (the special case $l = 2$ uses $A_1^{(1)} = 1$, so $2 \ell_2^2 A_1^{(1)} = 2 \ell_2^2 = A_1^{(2)}$). For $k = l$ the prefactor is $2 = A_l^{(l)}$ by definition. Substituting these identities into Eq.~\eqref{eq:ind-step} gives
\begin{equation}
\lVert \boldsymbol{\delta}_l(\x) \rVert_2^2 \;\leq\; \sum_{k=1}^{l} A_k^{(l)} \, \lVert \eu_k(\hat{f}_{k-1}(\x)) \rVert_2^2 ,
\end{equation}
which completes the induction. Setting $l = L$ and writing $A_k := A_k^{(L)}$ yields
\begin{equation}
\lVert \boldsymbol{\delta}_L(\x) \rVert_2^2 \;\leq\; \sum_{k=1}^{L} A_k \, \lVert \eu_k(\hat{f}_{k-1}(\x)) \rVert_2^2 .
\label{eq:pointwise}
\end{equation}

\paragraph{Expectation.}
Taking expectation over $\mathbf{X} \in \mathcal{D}_0$ on both sides of Eq.~\eqref{eq:pointwise},
\begin{equation}
\E_{\mathbf{X}} \, \lVert \boldsymbol{\delta}_L(\mathbf{X}) \rVert_2^2
\;\leq\;
\sum_{k=1}^{L} A_k \, \E_{\mathbf{X}} \, \lVert \eu_k(\hat{f}_{k-1}(\mathbf{X})) \rVert_2^2 .
\label{eq:exp-bound}
\end{equation}
For each $k$, the inner expectation is over the distribution of the network's intermediate input $\mathbf{X}_k := \hat{f}_{k-1}(\mathbf{X})$, which lies in $\mathcal{D}_{k-1}$ by the codomain-invariance part of Assumption~\ref{ass:lip}. Renaming this expectation as $\E_{\mathbf{X}_k}$ in Eq.~\eqref{eq:exp-bound} gives Eq.~\eqref{eq:deep-risk} of Theorem~\ref{thm:deep-risk}.

\paragraph{Reduction for known operators.}
If layer $k$ is a known operator, $\hat{u}_k = u_k$ on $\mathcal{D}_{k-1}$, hence $\eu_k \equiv \mathbf{0}$ and $\E_{\mathbf{X}_k} \lVert \eu_k(\mathbf{X}_k) \rVert_2^2 = 0$. Therefore the corresponding term of Eq.~\eqref{eq:deep-risk} vanishes.

\paragraph{Application of Assumption~\ref{ass:barron}.}
For each learned layer $k \in \mathcal{L}$, Assumption~\ref{ass:barron} bounds $\E_{\mathbf{X}_k} \lVert \eu_k(\mathbf{X}_k) \rVert_2^2$ by $C_k^2 / n_k + \kappa_k p_k \log N / N$. Substituting this into Eq.~\eqref{eq:deep-risk} and dropping the zero terms of known operator layers gives Eq.~\eqref{eq:deep-risk-barron}. \qed

\section{Derivation of the Lipschitz Composition Bound}
\label{app:lip}

Theorem~\ref{thm:deep-risk} uses the Lipschitz constants $\ell_l$ of the \emph{true} layers $u_l$. For completeness, we include the standard derivation of the multiplicative composition bound used in Section~\ref{sec:theory}.

\begin{lemma}
\label{lem:lip}
Under Assumption~\ref{ass:lip}, the Lipschitz constant of the composed function $f_l$ in Eq.~\eqref{eq:target} satisfies
\begin{equation}
\ell_{f_l} \;\leq\; \prod_{j=1}^{l} \ell_j .
\label{eq:lip-prod}
\end{equation}
\end{lemma}

\begin{proof}
For $\x, \tilde{\x} \in \mathcal{D}_0$,
\begin{align}
\lVert f_l(\x) - f_l(\tilde{\x}) \rVert_2
&= \lVert u_l(f_{l-1}(\x)) - u_l(f_{l-1}(\tilde{\x})) \rVert_2 \nonumber \\
&\leq \ell_l \, \lVert f_{l-1}(\x) - f_{l-1}(\tilde{\x}) \rVert_2 \nonumber \\
&\leq \ell_l \cdot \ell_{l-1} \cdot \lVert f_{l-2}(\x) - f_{l-2}(\tilde{\x}) \rVert_2 \nonumber \\
&\;\;\vdots \nonumber \\
&\leq \prod_{j=1}^{l} \ell_j \cdot \lVert \x - \tilde{\x} \rVert_2 ,
\end{align}
which gives Eq.~\eqref{eq:lip-prod}.
\end{proof}

\section{Application to the CT Reconstruction Network}
\label{app:ct}

We apply Theorem~\ref{thm:deep-risk} to the layered FBP reconstruction~\eqref{eq:fbp}. The four layers are $u_1 = \weights$, $u_2 = \filter$, $u_3 = \proj^\top$, and $u_4 = \mathrm{ReLU}$. The Lipschitz constants are
\begin{align}
\ell_1 &= \lVert \weights \rVert_2 , \quad
\ell_2 = \lVert \filter \rVert_2 , \quad
\ell_3 = \lVert \proj^\top \rVert_2 , \quad
\ell_4 = 1 ,
\end{align}
where the first three are operator norms of the corresponding linear maps and $\ell_4 = 1$ because $\mathrm{ReLU}$ is $1$-Lipschitz \cite{Maier2019}. With these constants the layer amplifications $A_l$ in Eq.~\eqref{eq:deep-risk} become
\begin{align}
A_1 &= 2^{L-1} \cdot \ell_2^2 \, \ell_3^2 \, \ell_4^2 = 8 \, \lVert \filter \rVert_2^2 \, \lVert \proj^\top \rVert_2^2 , \\
A_2 &= 2^{L-1} \cdot \ell_3^2 \, \ell_4^2 = 8 \, \lVert \proj^\top \rVert_2^2 , \\
A_3 &= 2^{L-2} \cdot \ell_4^2 = 4 , \\
A_4 &= 2 .
\end{align}
These values follow from the piecewise definition of $A_l$ in Eq.~\eqref{eq:Al}.

\paragraph{Operator-aware network.}
In the operator-aware network of \cite{Wuerfl2018,Maier2019}, only $\weights$ is learned; $\filter$, $\proj^\top$, and $\mathrm{ReLU}$ are known operators. By Theorem~\ref{thm:deep-risk}, the bound reduces to
\begin{equation}
\E_{\mathbf{X}} \lVert \y - \hat{\y}_{\mathrm{KO}} \rVert_2^2
\;\leq\;
A_1 \, \E_{\mathbf{X}_1} \lVert \eu_1(\mathbf{X}_1) \rVert_2^2
\;\leq\;
A_1 \!\left( \frac{C_1^2}{n_1} + \kappa_1 \, \frac{p_{\mathrm{KO}} \log N}{N} \right) ,
\end{equation}
with $p_{\mathrm{KO}} = V \cdot B = 180 \cdot 512 = 92{,}160$ trainable diagonal weights at the slice-wise setting of \citet{Wuerfl2018} and \citet{Maier2019} ($H = 512$, $V = 180$, $B = 512$) and $A_1 = 8 \, \lVert \filter \rVert_2^2 \, \lVert \proj^\top \rVert_2^2$.

\paragraph{Fully connected counterfactual.}
The fully connected counterfactual collapses the entire reconstruction pipeline into a single learned dense matrix $\mathbf{M}$ followed by a fixed ReLU, $\hat{\y}_{\mathrm{FC}} = \mathrm{ReLU}(\mathbf{M} \, \x)$. Its composition therefore has only two layers, $u_1^{\mathrm{FC}} = \mathbf{M}$ and $u_2^{\mathrm{FC}} = \mathrm{ReLU}$, with $\mathrm{ReLU}$ being the only known operator. The amplification factor for the learned layer is
\begin{equation}
A_1^{\mathrm{FC}} = 2^{L-1} \cdot \ell_2^2 = 2 ,
\end{equation}
since $\mathrm{ReLU}$ is $1$-Lipschitz and $L = 2$. By Theorem~\ref{thm:deep-risk}, the bound has only one learned-layer term:
\begin{equation}
\E_{\mathbf{X}} \lVert \y - \hat{\y}_{\mathrm{FC}} \rVert_2^2
\;\leq\;
A_1^{\mathrm{FC}} \!\left( \frac{C_{\mathbf{M}}^2}{n_{\mathbf{M}}} + \kappa_{\mathbf{M}} \, \frac{p_{\mathrm{FC}} \log N}{N} \right) ,
\end{equation}
with $p_{\mathrm{FC}} = N_{\mathrm{pixels}} \cdot N_{\mathrm{measurements}} = H^2 \cdot V \cdot B = 24{,}159{,}191{,}040$ trainable weights at the same $H = 512$ setting. The two amplification factors $A_1$ (KO) and $A_1^{\mathrm{FC}}$ (FC) are not equal in general: $A_1$ depends on the operator norms $\lVert \filter \rVert_2$, $\lVert \proj^\top \rVert_2$ and therefore on the geometry, while $A_1^{\mathrm{FC}} = 2$ is geometry-independent.

\paragraph{Network-level calibration constants.}
Grouping the layer-level constants of each architecture into the network-level floor and slope used by the calibrated bound of Section~\ref{sec:ct},
\begin{align}
\mathrm{floor}_{\mathrm{KO}} &:= A_1 \, C_1^2 / n_1 , & \sigma_{\mathrm{KO}} &:= A_1 \, \kappa_1 \, p_{\mathrm{KO}} , \\
\mathrm{floor}_{\mathrm{FC}} &:= A_1^{\mathrm{FC}} \, C_{\mathbf{M}}^2 / n_{\mathbf{M}} , & \sigma_{\mathrm{FC}} &:= A_1^{\mathrm{FC}} \, \kappa_{\mathbf{M}} \, p_{\mathrm{FC}} ,
\end{align}
the network-level bound for each architecture reads $\E\lVert \y - \hat{\y} \rVert_2^2 \leq \mathrm{floor} + \sigma \, \log N / N$.

\paragraph{Sample-complexity ratio at matched target error.}
Equating both bounds at the same target error $\varepsilon > \max(\mathrm{floor}_{\mathrm{KO}}, \mathrm{floor}_{\mathrm{FC}})$ and solving each side for $N / \log N$ yields the two-factor decomposition
\begin{equation}
\frac{N_{\mathrm{FC}} / \log N_{\mathrm{FC}}}{N_{\mathrm{KO}} / \log N_{\mathrm{KO}}} \;=\; \underbrace{\frac{\sigma_{\mathrm{FC}}}{\sigma_{\mathrm{KO}}}}_{\text{structural-sparsity factor}} \;\cdot\; \underbrace{\frac{\varepsilon - \mathrm{floor}_{\mathrm{KO}}}{\varepsilon - \mathrm{floor}_{\mathrm{FC}}}}_{\text{approximation-budget factor}} .
\label{eq:ct-twofactor}
\end{equation}

Under matched Barron-type constants $\kappa_1 = \kappa_{\mathbf{M}}$, the structural-sparsity factor reduces to
\begin{equation}
\frac{\sigma_{\mathrm{FC}}}{\sigma_{\mathrm{KO}}} \;=\; \frac{A_1^{\mathrm{FC}} \, p_{\mathrm{FC}}}{A_1 \, p_{\mathrm{KO}}} \;=\; \frac{H^2}{4 \, \lVert \filter \rVert_2^2 \, \lVert \proj^\top \rVert_2^2} ,
\label{eq:ct-structural}
\end{equation}
which is the parameter-count ratio $H^2$ folded together with the operator-norm correction $1/(4 \, \lVert \filter \rVert_2^2 \, \lVert \proj^\top \rVert_2^2)$ that the analytic FBP decomposition exposes. The approximation-budget factor in Eq.~\eqref{eq:ct-twofactor} is unity exactly when the two architectures share the same approximation floor; whenever the dense matrix's approximation budget at finite width is worse than the operator-aware diagonal weighting's, the factor is greater than one and grows as the target error $\varepsilon$ approaches the FC's higher floor. Both factors are directly observable from the calibration table of Section~\ref{sec:ct} (Table~\ref{tab:calibration}); inverting Eq.~\eqref{eq:ct-twofactor} therefore turns Theorem~\ref{thm:deep-risk} into a closed-form sample-budget predictor that demonstrates the importance of the empirical floor difference.

\newpage
%

\section{Experimental details and qualitative results}
\label{app:supp-exp}

The main paper reports sample-efficiency sweeps at five operating points: a CPU surrogate at $H \in \{8, 16, 32\}$ and full-resolution GPU sweeps at $H \in \{128, 256\}$. This appendix records the experimental details that did not fit into the body and presents qualitative reconstructions from the GPU sweeps.

The structure of this appendix is as follows. Section~\ref{sec:cpu-details} describes the CPU surrogate setup (phantom generator, closed-form ridge fits, hyperparameter selection). Section~\ref{sec:gpu-details} describes the GPU sweeps (training protocol, dataset, evaluation metric). Section~\ref{sec:qualitative} shows representative reconstructions at $H = 128$ and $H = 256$ as a sanity check on the calibrated quantities reported in the main paper. Section~\ref{sec:supp-conclusion} closes with regime observations connecting the experimental data back to Proposition~\ref{prop:ct-proxy} of the main paper.

\subsection{CPU surrogate details}
\label{sec:cpu-details}

The CPU surrogate of Section~\ref{sec:exp} uses random-ellipse phantoms at three image scales $H \in \{8, 16, 32\}$ with $V = 1.25\,H$ equally spaced views and $B = H$ detector bins. For each phantom in the pool we draw between one and four ellipses with random amplitudes in $[0.2, 1.0]$, random centres in $[-0.5, 0.5]^2$, random axis lengths in $[0.08, 0.4]^2$, and a random orientation; pixel intensities are clipped to $[0, 1.5]$.

The forward projection is implemented as a sum-along-rotated-rows operator computed in pixel space: for each view $\theta$, the image is bilinearly rotated and summed along the first axis. This produces the discrete forward matrix $\proj \in \mathbb{R}^{(V \cdot B) \times H \cdot H}$. The fixed analytic inverse used by the operator-aware network is the Tikhonov-regularised pseudo-inverse $\proj^{+} = (\proj^\top \proj + 0.1\, I)^{-1} \proj^\top$, with the regularisation strength $0.1$ chosen so that $\proj^\top \proj + 0.1\, I$ is well-conditioned at every $H$. The fully connected baseline learns a single dense projection-to-image map $\mathbf{M} \in \mathbb{R}^{(H \cdot H) \times (V \cdot B)}$ with no fixed structure.

Both architectures are fit in closed form by ridge-regularised least squares. For each $(H, N, \mathrm{seed})$ cell the regularisation coefficient is selected from the grid $\lambda \in \{10^{-6}, 10^{-4}, 10^{-2}, 10^{0}, 10^{2}\}$ on a held-out validation set of $32$ phantoms; the test MSE is then evaluated on a separate test set of $128$ phantoms. We sweep $N \in \{4, 8, 16, 32, 64\}$ and average over five random seeds.

The complete CPU surrogate sweep takes about $14$~s on a recent x86 laptop. Per-seed test MSE values, the calibration constants of Table~\ref{tab:calibration}, and the script that produces the corresponding panels of Figure~\ref{fig:sweeps} are included in the release that accompanies the paper.

\subsection{GPU sweep details}
\label{sec:gpu-details}

The GPU sweeps at $H \in \{128, 256\}$ use the same random-ellipse phantom population as the CPU surrogate, regenerated at the higher resolution. The geometry settings are $V = 60$ views and $B = 128$ detector bins at $H = 128$, and $V = 90$ views and $B = 256$ detector bins at $H = 256$, both with a $180^{\circ}$ angular range. We average over three random seeds at every $(H, N)$ cell.

\paragraph{Training protocol.}
The operator-aware network is trained with stochastic gradient descent for $5{,}000$ iterations at learning rate $2 \cdot 10^{-2}$ and batch size $4$ at $H = 128$ (batch size $1$ at $H = 256$, reflecting the higher per-step memory cost). The fully connected baseline is fit by ridge regression over the same lambda grid as the CPU surrogate, with $\lambda$ selected by hold-out validation on a separate $50$-phantom validation set. Both architectures share the same per-seed phantom pool, so the comparison is conducted on identical data.

\paragraph{Sample sizes.}
We sweep $N \in \{4, 16, 64, 256, 1024, 2048\}$. The sweep covers four orders of magnitude in $N$ and lets us observe both the small-sample regime (where the operator-aware network's structural prior pays off) and the large-sample regime (where the fully connected baseline's representational flexibility eventually closes the gap at $H = 128$).

\paragraph{Evaluation metric.}
Test error is reported as relative root-mean-squared error (rRMSE) on a held-out test set of $50$ phantoms, $\mathrm{rRMSE} = \lVert \y - \hat\y \rVert_2 / \lVert \y \rVert_2$. The calibrated bound from Eq.~\eqref{eq:calibrated-bound} has the same shape regardless of whether the metric on the y-axis is MSE or rRMSE$^2$, with a different overall scaling absorbed into $\mathrm{floor}$ and $\sigma$. We use rRMSE$^2$ for the GPU panels of Figure~\ref{fig:sweeps}.

\paragraph{Hardware.}
The GPU sweeps were executed on a shared compute cluster. The operator-aware sweeps at $H \in \{128, 256\}$ ran on a single NVIDIA Tesla V100 (16~GiB) per job. The fully connected baseline at $H = 256$ requires more weight memory than fits on a single V100 and was therefore trained on four NVIDIA RTX~6000 (24~GiB each) with PyTorch fully-sharded data-parallel (FSDP); the fully connected baseline at $H = 128$ fits on a single V100 and uses the same single-GPU configuration as the operator-aware sweeps. End-to-end wall-clock time for a single $(H, N, \mathrm{seed})$ cell ranges from a few minutes at the smallest sample sizes to a few hours at $N = 2{,}048$, with the four-GPU FSDP runs at $H = 256$ at the upper end of this range.

\subsection{Qualitative reconstructions}
\label{sec:qualitative}

Figures~\ref{fig:gpu128-recons} and~\ref{fig:gpu256-recons} show two representative test phantoms at each operating point along with the reconstructions produced by the operator-aware and fully connected networks at every sample size. The KO reconstructions match the phantom outline at $N = 4$ and approach the ground truth by $N = 64$. The FC reconstructions are noisier at small $N$ but close the gap as $N$ grows; at $H = 128$ the FC eventually matches the KO visually by $N = 2048$, while at $H = 256$ a residual blur is still visible at the largest sample size.

\begin{figure}[h]
\centering
\includegraphics[width=\linewidth]{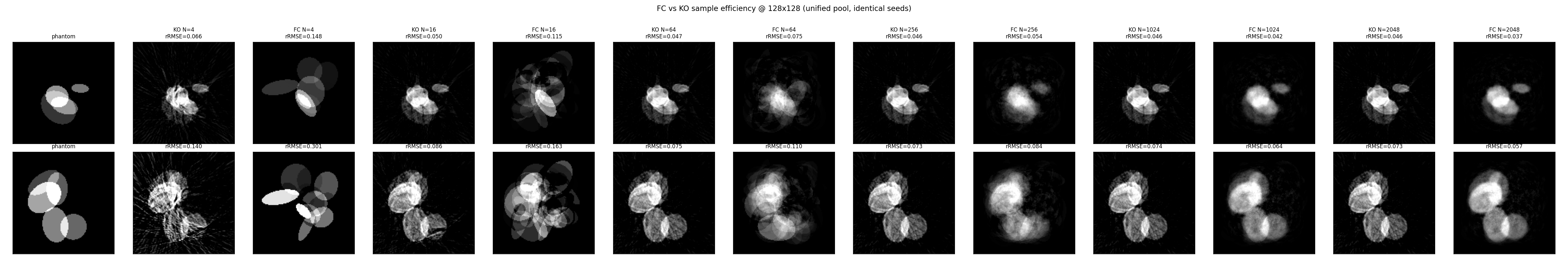}
\caption{Reconstructions at $H = 128$ for the unified phantom pool with identical seeds. Two test phantoms (rows) under the operator-aware and fully connected networks at six sample sizes $N$. The KO reconstruction is already faithful at $N = 4$; the FC catches up around $N = 1024$ -- $2048$.}
\label{fig:gpu128-recons}
\end{figure}

\begin{figure}[h]
\centering
\includegraphics[width=\linewidth]{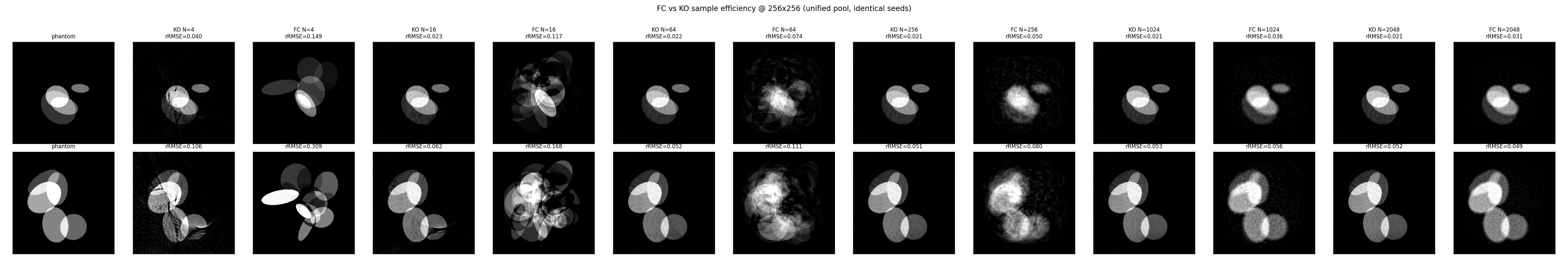}
\caption{Reconstructions at $H = 256$ for the unified phantom pool with identical seeds. Two test phantoms (rows) under the operator-aware and fully connected networks at six sample sizes $N$. The KO reconstruction matches the phantom from $N = 4$ onwards; the FC remains visibly blurred even at $N = 2048$, consistent with the larger approximation-budget gap reported in Table~\ref{tab:calibration}.}
\label{fig:gpu256-recons}
\end{figure}

The corresponding rRMSE-vs-$N$ curves with calibrated-bound overlays are reproduced in the main paper as the bottom row of Figure~\ref{fig:sweeps}.

\subsection{Discussion of regimes observed across the five sweeps}
\label{sec:supp-conclusion}

Across the five operating points the calibrated bound holds as an upper bound on the empirical curve in the data-driven regime of every sweep, and the floor difference between the two architectures plays the role predicted by Proposition~\ref{prop:ct-proxy}. Three regime observations are worth recording.

First, on the CPU surrogate the FC's approximation floor is consistently above the KO's, and the approximation-budget factor of Eq.~\eqref{eq:prop2-twofactor} is therefore consistently greater than one. The FC needs more samples to reach any target error than the KO at the same image scale.

Second, on the GPU sweep at $H = 128$ the FC actually crosses the KO floor at the largest sample size in our grid: at $N = 2{,}048$ the FC reaches rRMSE $\approx 4.7 \cdot 10^{-2}$ while the KO has saturated near $5.5 \cdot 10^{-2}$. The bound's approximation-budget factor in this regime is below one, which the calibration in Table~\ref{tab:calibration} directly reports ($\mathrm{floor}_{\mathrm{FC}} < \mathrm{floor}_{\mathrm{KO}}$ for the $H = 128$ GPU row). This shows that the operator-aware advantage is a sample-complexity statement, not an accuracy ceiling: with enough data the FC's larger parameter count gives it an asymptotic edge, but the operator-aware network reaches the same accuracy with $\sim 30 \times$ fewer samples in this case.

Third, on the GPU sweep at $H = 256$ the FC's floor is again above the KO's within the $N \leq 2{,}048$ range, and inverting the calibrated bound predicts that the FC needs roughly an order of magnitude more samples than the KO to reach the KO's plateau. At realistic clinical-CT sample budgets ($N$ in the low thousands of slices) the operator-aware decomposition therefore remains the data-efficient choice across every scale we tested, with the gap most pronounced at the largest image sizes.

\paragraph{Parameter-count ratio versus slope ratio.}
The parameter ratio $p_{\mathrm{FC}}/p_{\mathrm{KO}} = H^2$ should not be read as a sample-complexity ratio. As an architectural cost, $H^2$ is a hard memory and compute multiplier: the dense substitute simply does not fit on accessible single-GPU hardware at $H = 512$ even before any training-set question is raised. As a sample-complexity proxy, however, $H^2$ overstates the gap; the slope ratio $\sigma_{\mathrm{FC}}/\sigma_{\mathrm{KO}}$ from Eq.~\eqref{eq:prop2-twofactor} sits between two and ten across our sweeps and is the empirically tighter predictor of how much extra data the dense network actually needs.

\paragraph{Low-data regime versus sparse-view expressivity.}
The operator-aware network wins clearly in the low-data regime --- it reaches its asymptote with roughly an order of magnitude fewer samples than the dense baseline at every scale we test --- but with sufficient data the dense matrix's expressivity catches up. At the $H = 128$ sparse-view operating point the dense matrix surpasses the operator-aware floor at $N = 2{,}048$ (Figure~\ref{fig:gpu128-recons}). The known-operator FBP architecture of \citet{Maier2019} was designed for limited-angle CT, where the missing data is concentrated outside the measured angular range, rather than for sparse-view CT, where the missing data is distributed across all angles; that a dense matrix can find a better sparse-view solution once it has been given enough data is consistent with this task-domain mismatch and is directly visible in the qualitative reconstructions of Section~\ref{sec:qualitative}.

\end{document}